\documentclass[letterpaper,11pt]{article}
\usepackage[toc,page]{appendix}
\usepackage[margin=1in]{geometry}
\usepackage[bookmarks, colorlinks=true, plainpages = false, citecolor = blue,linkcolor=red,urlcolor = blue, filecolor = blue]{hyperref}
\usepackage{url}
\usepackage{amsmath,amsfonts,amsthm,amssymb,bm, verbatim,dsfont,mathtools}
\usepackage{color,graphicx}
\usepackage{subfigure}
\usepackage{etoolbox}
\usepackage{tikz}
\usepackage{xr,xspace}
\usepackage{todonotes}
\usepackage{paralist}
\usepackage{caption,soul}
\usepackage{algorithm}
\usepackage{algorithmic}
\makeatletter

\theoremstyle{plain}
\newtheorem{theorem}{Theorem}
\newtheorem{lemma}{Lemma}

\theoremstyle{definition}

\newtheorem{remark}{Remark}
\newtheorem*{remark*}{Remark}

\newcommand \E[1]{\mathbb{E}[#1]}
\newcommand{\Perp}{\perp}

\newcommand{\argmax}{\mathop{\arg\max}}
\newcommand{\argmin}{\mathop{\arg\min}}

\usepackage{xspace,prettyref}
\newcommand{\diverge}{\to\infty}

\newcommand{\iiddistr}{{\stackrel{\text{\iid}}{\sim}}}

\newcommand{\zeros}{\mathbf 0}
\newcommand{\reals}{{\mathbb{R}}}

\newcommand{\naturals}{{\mathbb{N}}}

\newcommand{\eexp}{{\rm e}}

\newcommand{\identity}{\mathbf I}
\newcommand{\allones}{\mathbf J}

\newcommand{\diff}{{\rm d}}

\newcommand{\Expect}{\mathbb{E}}
\newcommand{\expect}[1]{\mathbb{E}\left[ #1 \right]}
\newcommand{\eexpect}[1]{\mathbb{E}[ #1 ]}

\newcommand{\Prob}{\mathbb{P}}
\newcommand{\pprob}[1]{ \mathbb{P}\{ #1 \} }
\newcommand{\prob}[1]{ \mathbb{P}\left\{ #1 \right\} }

\newcommand{\Bern}{{\rm Bern}}
\newcommand{\Binom}{{\rm Binom}}

\newcommand{\eg}{e.g.\xspace}
\newcommand{\ie}{i.e.\xspace}
\newcommand{\iid}{i.i.d.\xspace}
\newrefformat{eq}{(\ref{#1})}
\newrefformat{chap}{Chapter~\ref{#1}}
\newrefformat{sec}{Section~\ref{#1}}
\newrefformat{algo}{Algorithm~\ref{#1}}
\newrefformat{fig}{Fig.~\ref{#1}}
\newrefformat{tab}{Table~\ref{#1}}
\newrefformat{rmk}{Remark~\ref{#1}}
\newrefformat{clm}{Claim~\ref{#1}}
\newrefformat{def}{Definition~\ref{#1}}
\newrefformat{cor}{Corollary~\ref{#1}}
\newrefformat{lmm}{Lemma~\ref{#1}}
\newrefformat{prop}{Proposition~\ref{#1}}
\newrefformat{app}{Appendix~\ref{#1}}
\newrefformat{hyp}{Hypothesis~\ref{#1}}
\newrefformat{thm}{Theorem~\ref{#1}}

\newcommand{\pth}[1]{\left( #1 \right)}

\newcommand{\sth}[1]{\left\{ #1 \right\}}

\newcommand{\iprod}[2]{\left \langle #1, #2 \right\rangle}
\newcommand{\Iprod}[2]{\langle #1, #2 \rangle}
\newcommand{\indc}[1]{{\mathbf{1}_{\left\{{#1}\right\}}}}

\newcommand{\diag}[1]{\mathsf{diag} \left\{ {#1} \right\} }

\newcommand{\tY}{{\widetilde{Y}}}
\newcommand{\tZ}{{\widetilde{Z}}}

\newcommand{\calG}{{\mathcal{G}}}

\newcommand{\calS}{{\mathcal{S}}}

\newcommand{\calY}{{\mathcal{Y}}}
\newcommand{\calZ}{{\mathcal{Z}}}

\newcommand{\ML}{{\rm ML}\xspace}
\newcommand{\SDP}{{\rm SDP}\xspace}

\newcommand{\ER}{Erd\H{o}s-R\'enyi\xspace}

\newcommand{\Tr}{\mathsf{Tr}}

\begin{document}

\title{Achieving Exact Cluster Recovery Threshold via Semidefinite Programming}

\date{\today}

\author{Bruce Hajek \and Yihong Wu \and Jiaming Xu\thanks{
This paper was accepted to IEEE Transactions on Information Theory on January 3, 2016.
The material in this paper was presented in part  at the IEEE International Symposium on Information Theory, Hong Kong, June, 2015 \cite{HWX14b-isit}.} 
\thanks{This research was supported by the National Science Foundation under
Grant CCF 14-09106, IIS-1447879, NSF OIS 13-39388, and CCF 14-23088, and
Strategic Research
Initiative on Big-Data Analytics of the College of Engineering
at the University of Illinois, and DOD ONR Grant N00014-14-1-0823, and Grant 328025 from the Simons Foundation.}
\thanks{B. Hajek and Y. Wu are with
the Department of ECE and Coordinated Science Lab, University of Illinois at Urbana-Champaign, Urbana, IL, \texttt{\{b-hajek,yihongwu\}@illinois.edu}.
J. Xu is with Department of Statistics, The Wharton School, University of Pennsylvania, Philadelphia, PA, \texttt{jiamingx@wharton.upenn.edu}. }
}

\maketitle
\begin{abstract}
The binary symmetric stochastic block model deals with a random graph of $n$ vertices partitioned into
two equal-sized clusters, such that each pair of vertices is connected independently with probability $p$ within clusters and $q$
across clusters. In the asymptotic regime of $p=a \log n/n$ and $q=b \log n/n$ for fixed $a,b$ and $n\diverge$, we show that the semidefinite programming relaxation of the maximum likelihood estimator achieves the optimal threshold for exactly recovering the partition from the graph with probability tending to one, resolving a conjecture of  Abbe et al.\ \cite{Abbe14}. Furthermore, we show that the semidefinite programming relaxation also achieves the optimal recovery threshold in the planted dense subgraph model containing a single cluster of size proportional to $n$.
\end{abstract}

\section{Introduction}
The community detection problem refers to finding the underlying communities within a network using only the knowledge of the network topology \cite{Fortunato10,Newman04}.
This paper considers the following probabilistic model for generating a network with underlying community structures:
Suppose that out of a total of $ n $ vertices,
$ rK $ of them are partitioned into $ r $ {clusters} of size $ K $,
and the remaining $ n-rK $ vertices do not belong to any clusters (called outlier vertices);
a random graph $G$ is generated based on the cluster structure, where each pair of vertices is connected independently with probability $p$
if they are in the same cluster or $q$ otherwise. In particular, an outlier vertex is connected to any other vertex with probability $q.$
This random graph ensemble is known as the \emph{planted cluster model} \cite{ChenXu14}
 with parameters $ n,r,K \in \naturals$ and $p,q \in [0,1]$ such that $n \geq rK$.
In particular, we call $p$ and $q$ the in-cluster and cross-cluster edge density, respectively.
The planted cluster model encompasses several classical planted random graph models including planted clique~\cite{Alon98},
 planted coloring~\cite{alon1997coloring3},
 planted dense subgraph~\cite{arias2013community}, planted partition~\cite{Condon01}, and the stochastic block model~\cite{Holland83},
which have been widely used for studying the community detection and graph partitioning problem (see, e.g.,~\cite{McSherry01,Decelle11,Mossel12,Chen12}
and the references therein). 

In this paper, we focus on the following particular cases:
 \begin{itemize}
\item  Binary symmetric stochastic block model (assuming $n$ is even):
\begin{align}
r=2,\,\, K=\frac{n}{2},\,\, p=\frac{a\log n}{n},\,\, q= \frac{b \log n}{n},\,\, n \to \infty .\label{eq:SBMscaling}
\end{align}
 \item Planted dense subgraph model:
 \begin{align}
r=1,\,\,  K= \lfloor \rho n \rfloor,\,\, p=\frac{a\log n}{n},\,\, q= \frac{b \log n}{n},\,\, n \to \infty,\,\, \label{eq:PDSscaling}
 \end{align}
 \end{itemize}
 where $a \neq b$ and $0<\rho<1$ are fixed constants, and study the problem of exactly recovering the clusters
 (up to a permutation of cluster indices) from the observation of the graph $G$. 

 Exact cluster recovery under the binary symmetric stochastic block model
 is studied in \cite{Abbe14,Mossel14} and a sharp recovery threshold is found.
\begin{theorem}[\cite{Abbe14,Mossel14}]
Under the binary symmetric stochastic block model \prettyref{eq:SBMscaling}, if $( \sqrt{a}-\sqrt{b} )^2 >  2$,\footnote{If $b>0$,
$(\sqrt{a}-\sqrt{b} )^2=2$ is also sufficient for exact recovery as shown by \cite{Mossel14}. But for $b=0$, since $a=2$, the ER random graph $G(n/2, 2\log n/n)$ contains
isolated vertices with probability bounded away from zero and exact recovery is impossible.}
the clusters can be exactly recovered up to a permutation of cluster indices with probability converging to one; if $( \sqrt{a}-\sqrt{b} )^2< 2$,
no algorithm can exactly recover the clusters with probability converging to one.
\label{thm:optimalSBM}
\end{theorem}

The optimal reconstruction threshold in \prettyref{thm:optimalSBM} is achieved by the maximum likelihood (ML) estimator,
which entails finding the minimum bisection of the graph, a problem known to be NP-hard in the worst case \cite[Theorem 1.3]{garey76}.
Nevertheless, it has been shown that the optimal recovery threshold can be attained in polynomial time
using a two-step procedure \cite{Abbe14,Mossel14}: First, apply the partial recovery algorithms
in \cite{Mossel13,Massoulie13} to
correctly cluster all but $o(n)$ vertices; Second, flip the cluster memberships of those vertices who do not agree with the
majority of their neighbors.
It remains open to find a simple direct approach to achieve the exact recovery threshold in polynomial time.
It was proved in \cite{Abbe14} that a semidefinite programming (SDP) relaxation of the ML estimator succeeds if $(a-b)^2>8(a+b)+\frac{8}{3}(a-b)$.
Backed by compelling simulation results, it was
further conjectured in \cite{Abbe14} that the SDP relaxation can achieve the
optimal recovery threshold.
In this paper, we resolve this conjecture in the positive.

In addition, we prove that the SDP relaxation achieves the optimal recovery threshold for the
planted dense subgraph model \prettyref{eq:PDSscaling} where the cluster size $K$ scales \emph{linearly} in $n$.
This conclusion is in sharp contrast to the following computational barrier established in \cite{HajekWuXu14}:
If $K$ grows and $p,q$ decay \emph{sublinearly} in $n$,
attaining the  statistical optimal recovery threshold is at least as hard as solving the planted clique problem
(See \prettyref{sec:pds} for detailed discussions).

Since the initial posting of this paper to arXiv, a number of interesting papers have been posted, some
extending or improving our results. Another
resolution of the conjecture in \cite{Abbe14} was  given in \cite{Bandeira15} independently.
A sharp characterization of the threshold for exact recovery for a general class of stochastic block models is derived in \cite{AbbeSandon15},
which includes the two cases considered in this paper as special cases.
Extensions of this paper appear in \cite{HajekWuXuSDP15},
showing SDP provides exact recovery up to the information theoretic threshold for a fixed number of equal sized clusters
or two unequal sized clusters.  More recently, the preprint \cite{ABBK} shows similar optimality results of SDP for
$o(\log n)$ number of equal-sized clusters and  \cite{perry2015semidefinite} establishes optimality results of SDP for
a fixed number of  clusters with unequal  sizes.

\paragraph{Notation}
Let $A$ denote the adjacency matrix of the graph $G$, $\identity $ denote the identity matrix,
and $\allones$ denote the all-one matrix.
We write  $X \succeq 0$ if $X$ is positive semidefinite and $X \ge 0$ if all the entries of $X$ are non-negative.
Let $\calS^n$ denote the set of all $n \times n$ symmetric matrices. For $X \in \calS^n$, let $\lambda_2(X)$ denote its second smallest eigenvalue.
For any matrix $Y$, let $\|Y\|$ denote its spectral norm.
For any positive integer $n$, let $[n]=\{1, \ldots, n\}$.
For any set $T \subset [n]$, let $|T|$ denote its cardinality and $T^c$ denote its complement.
We use standard big $O$ notations,
e.g., for any sequences $\{a_n\}$ and $\{b_n\}$, $a_n=\Theta(b_n)$ or $a_n  \asymp b_n$
if there is an absolute constant $c>0$ such that $1/c\le a_n/ b_n \le c$.
Let $\Bern(p)$ denote the Bernoulli distribution with mean $p$ and
$\Binom(n,p)$ denote the binomial distribution with $n$ trials and success probability $p$.
All logarithms are natural and we use the convention $0 \log 0=0$.

\section{Stochastic block model}\label{sec:SBM}
The cluster structure  under the binary symmetric stochastic block model can be represented by a vector
$\sigma \in \{\pm 1\}^n$ such that $\sigma_i=1$ if vertex $i$ is
in the first cluster and $\sigma_i=-1$ otherwise. Let $\sigma^\ast$ correspond to the true clusters.
Then the \ML estimator of $\sigma^\ast$  for the case $a>b$ can be simply stated as
\begin{align}
\max_{\sigma}  & \; \sum_{i,j} A_{ij} \sigma_i \sigma_j \nonumber  \\
\text{s.t.	}  & \; \sigma_i \in \{\pm1\}, \quad i \in [n] \nonumber \\
 & \; \sigma^\top \mathbf{1} =0,
 \label{eq:SBMML1}
\end{align}
which maximizes the number of in-cluster edges minus the number of out-cluster edges. This is equivalent to solving the NP-hard minimum graph bisection problem.
Instead, let us consider its convex relaxation similar to the
\SDP relaxation studied in \cite{Goemans95,Abbe14}.
Let $Y=\sigma \sigma^\top$. Then $Y_{ii}=1$ is equivalent to $\sigma_i = \pm 1$ and $\sigma^\top \mathbf{1}=0$
if and only if $\Iprod{Y}{\allones}=0$. Therefore, \prettyref{eq:SBMML1} can be recast as
\begin{align}
\max_{Y,\sigma}  & \; \Iprod{A}{Y} \nonumber  \\
\text{s.t.	} & \; Y=\sigma \sigma^\top  \nonumber  \\
& \;  Y_{ii} =1, \quad i \in [n]\nonumber \\
& \;  \Iprod{\allones}{Y} =0 . \label{eq:SBMML2}
\end{align}
Notice that the matrix $Y=\sigma \sigma^\top$ is a rank-one positive semidefinite matrix. If we relax this
condition by dropping the rank-one restriction, we obtain the following convex relaxation of \prettyref{eq:SBMML2},
which is a semidefinite program:
\begin{align}
\widehat{Y}_{\SDP} = \argmax_{Y}  & \; \langle A, Y \rangle \nonumber  \\
\text{s.t.	} & \; Y \succeq 0  \nonumber \\
 & \;  Y_{ii} =1, \quad i \in [n] \nonumber \\
 & \; \Iprod{\allones}{Y} =0. \label{eq:SBMconvex}
\end{align}
We remark that \prettyref{eq:SBMconvex} does not rely on any knowledge of the model parameters except that $a > b$;
for the case $a<b$, we replace $\argmax$ in \prettyref{eq:SBMconvex} by $\argmin$. The SDP formulation introduced in \cite{Abbe14} is
slightly different from ours: it did not impose the constraint $\Iprod{\allones}{Y} =0$ and the objective function is the inner product
between a \emph{weighted} adjacency matrix and $Y.$

Let $Y^\ast=\sigma^\ast (\sigma^\ast)^\top$ and $\calY_n \triangleq \{\sigma \sigma^\top: \sigma \in \{-1,1\}^n, \sigma^\top \mathbf{1} =0 \}$.
The following result establishes the optimality of the \SDP procedure:
\begin{theorem}\label{thm:SBMSharp}
If $ ( \sqrt{a}-\sqrt{b} )^2> 2$, then $\min_{Y^* \in \calY_n} \pprob{\widehat{Y}_{\SDP}=Y^\ast} = 1 - n^{-\Omega(1)}$ as $n \to \infty$.
\end{theorem}


\begin{remark}
It is worthy to note the relationship between our SDP relaxation and other related formulations that appeared previously in the literature.
In fact, \prettyref{eq:SBMconvex} coincides with the SDP studied in \cite[p. 659]{FK01} for MIN BISECTION, where a sufficient condition is obtained in
\cite[Lemma 19]{FK01} that is not optimal.
	The SDP relaxation considered in \cite[Equation (15)]{FriezeJerrum97} for MAX BISECTION is equivalent to \prettyref{eq:SBMconvex} with max replaced by min (used when $a<b$) and $\Iprod{\allones}{Y} =0$ 	by $\Iprod{\allones}{Y}  \leq 0$.	 The proof of \prettyref{thm:SBMSharp} shows that this more relaxed version
	also works. 	
Finally, we note that the SDP used in \cite{Abbe14} can be viewed as a penalized version of \prettyref{eq:SBMconvex} with $-\frac{1}{2} \Iprod{\allones}{Y}$ added to the objective function.
	\label{rmk:sdp}
\end{remark}

\begin{remark}
\prettyref{thm:SBMSharp} naturally extends to the \emph{semirandom model} considered in \cite{FK01}, where after a graph is instantiated from the SBM, a \emph{monotone adversary} can add in-cluster edges and delete cross-cluster edges arbitrarily. Although this process may appear to make the cluster structure more visible, such an adversarial model is known to foil many procedures based on degrees, local search or graph spectrum \cite{FK01}. By design, the SDP \prettyref{eq:SBMconvex} enjoys robustness against such a monotone adversary, which has already been observed in \cite{FK01} and proved in \cite[Lemma 1]{Chen12}. More specifically, let $\widetilde{A}$ denote the adjacency matrix of the altered graph. Whenever $Y^\ast$ is the unique maximizer of \prettyref{eq:SBMconvex}, \ie, $\Iprod{A}{Y^*} > \Iprod{A}{Y}$ for any other feasible $Y$, then $Y^*$ is also the unique maximizer of \prettyref{eq:SBMconvex} with $A$ replaced by $\widetilde{A}$.
To see this, note that, by Cauchy-Schwartz inequality, $Y \succeq 0$ and $Y_{ii}=1$ imply that $|Y_{ij}| \leq 1$ for all $i,j$. Consequently, $\Iprod{\widetilde{A}-A}{Y} \leq \sum |\widetilde{A}_{ij}-A_{ij}| = \Iprod{\widetilde{A}-A}{Y^*}$. Then
$\Iprod{\widetilde{A}}{Y} = \Iprod{A}{Y} + \Iprod{\widetilde{A}-A}{Y} < \Iprod{A}{Y^*} + \Iprod{\widetilde{A}-A}{Y^*} =  \Iprod{\widetilde{A}}{Y^*}$, establishing the unique optimality of $Y^*$.
 	\label{rmk:semirandom}
\end{remark}

\section{Planted dense subgraph model}\label{sec:pds}

In this section we turn to the planted dense subgraph model in the asymptotic regime \prettyref{eq:PDSscaling}, where there exists a single cluster of size $\lfloor \rho n \rfloor$. To specify the optimal reconstruction threshold, define the following function: For $a,b \geq 0$, let
\begin{align}
f(a,b) =\left\{
 \begin{array}{rl}
   a - \tau^\ast \log \frac{\eexp a}{\tau^\ast}  & \text{if } a, b>0, a \neq b\\
   a & \text{if } b=0 \\
    b & \text{if }  a=0 \\
    0 & \text{if } a=b
    \end{array} \right., \label{eq:definitionthreshold}
\end{align}
where $\tau^\ast\triangleq\frac{a-b}{\log a- \log b}$ if $a,b>0$ and $ a \neq b$.
We show that if $\rho f(a,b)>1$, exact recovery is achievable in polynomial-time
via \SDP with probability tending to one; if $\rho f(a,b)<1$, any estimator fails to recover the cluster with probability tending to one regardless of the computational costs.
The sharp threshold $\rho f(a,b)=1$ is plotted in Fig.~\ref{fig:phasetransition} for various values of $\rho$.
\begin{figure}[ht]
\centering
\includegraphics[width=4in]{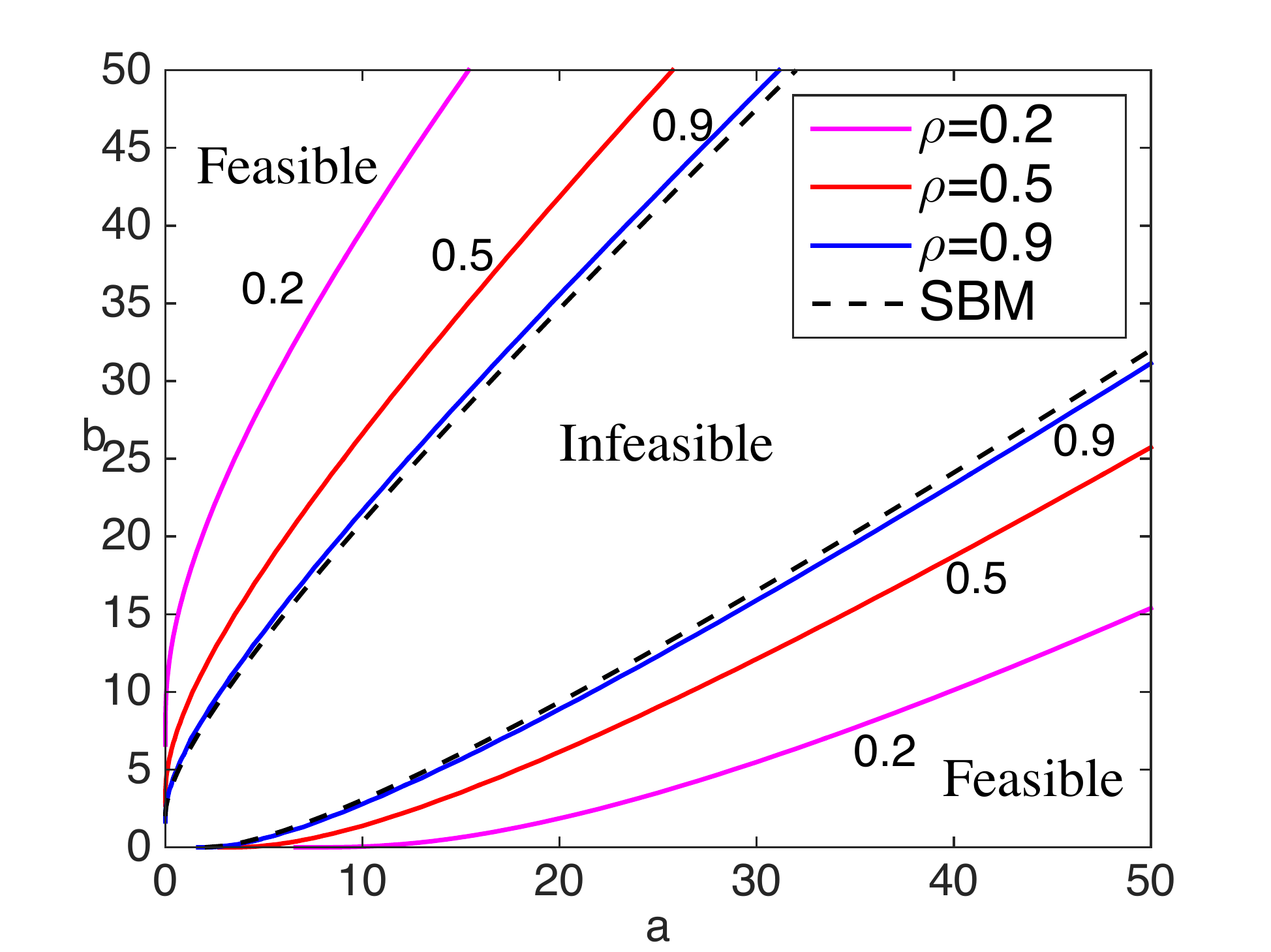}
\caption{The solid curves show the recovery threshold $ \{(a,b): \rho f(a,b) = 1\}$ for the planted dense subgraph model \prettyref{eq:PDSscaling} for three values of $\rho.$
For each $\rho$ there are two curves; recovery is not possible for $(a,b)$ in the open region between the curves and recovery is possible by the SDP for $(a,b)$
in the open region outside the two curves.
Similarly, the two dashed curves correspond to the recovery threshold $\{(a,b): (\sqrt{a} - \sqrt{b})^2= 2\}$ for the stochastic block model \prettyref{eq:SBMscaling}.}
\label{fig:phasetransition}
\end{figure}

We first introduce the maximum likelihood estimator and its convex relaxation.
For ease of notation, in this section we use a  vector $\xi \in \{0,1\}^n$, as opposed to $\sigma \in \{\pm 1\}^n$ used in \prettyref{sec:SBM} for the SBM, as the indicator function of the cluster, such that $\xi_i=1$ if vertex $i$ is in the cluster and $\xi_i=0$ otherwise.
Let $\xi^\ast$ be the indicator of the true cluster.
Assuming $a>b$, \ie, the vertices in the cluster are more densely connected, the \ML estimation of $\xi^\ast$ is simply
\begin{align}
\max_{\xi}  & \; \sum_{i,j} A_{ij} \xi_i \xi_j \nonumber  \\
\text{s.t.	}  & \; \xi \in \{ 0, 1 \}^n \nonumber \\
 & \; \xi^\top \mathbf{1} =K,
 \label{eq:PDSML1}
\end{align}
which maximizes the number of in-cluster edges.
Due to the integrality constraints, it is computationally difficult to solve  \prettyref{eq:PDSML1}, which prompts us to consider its convex relaxation.
Note that \prettyref{eq:PDSML1} can be equivalently\footnote{Here \prettyref{eq:PDSML1} and \prettyref{eq:PDSML2} are equivalent in the following sense: for any feasible $\xi$ for \prettyref{eq:PDSML1},
$Z=\xi\xi^\top$ is feasible for \prettyref{eq:PDSML2}; for any feasible $Z, \xi$ for \prettyref{eq:PDSML2}, either $\xi$ or $-\xi$
is feasible for \prettyref{eq:PDSML1}. } formulated  as
\begin{align}
\max_{Z,\xi}  & \; \Iprod{A}{Z} \nonumber  \\
\text{s.t.	} & \; Z=\xi \xi^\top  \nonumber  \\
& \;  Z_{ii} \le 1, \quad \forall i \in [n]\nonumber \\
 & \; Z_{ij} \ge 0,  \quad \forall i, j \in [n]  \nonumber \\
 & \;  \Iprod{\identity}{Z} = K \nonumber  \\
 & \; \Iprod{\allones}{Z} = K^2,
\label{eq:PDSML2}
\end{align}
where the matrix $Z=\xi \xi^\top$ is positive semidefinite and rank-one. Removing the rank-one restriction leads to the following convex relaxation of \prettyref{eq:PDSML2},
which is a semidefinite program.
\begin{align}
\widehat{Z}_{\SDP} = \argmax_{Z}  & \langle A, Z \rangle  \nonumber \\
\text{s.t.	} & \; Z \succeq 0   \nonumber \\
& \;  Z_{ii} \le 1, \quad \forall i \in [n]\nonumber \\
 & \; Z_{ij} \ge 0,  \quad \forall i, j \in [n]  \nonumber \\
& \;  \Iprod{\identity}{Z} = K  \nonumber  \\
 & \; \Iprod{\allones}{Z} = K^2. \label{eq:PDSCVX}
\end{align}
We note that, apart from the assumption that $a > b$, the only model parameter needed by the estimator \prettyref{eq:PDSCVX} is the cluster size $K$;
for the case $a<b$, we replace  $\argmax$ in \prettyref{eq:PDSCVX} by $\argmin$.

Let $Z^*=\xi^*(\xi^*)^\top$ correspond to the true cluster and define $\calZ_n = \big\{\xi \xi^\top: \xi \in \{ 0, 1 \}^n , \xi^\top \mathbf{1} =K\big\}$. The recovery threshold for the SDP \prettyref{eq:PDSCVX} is given as follows.
\begin{theorem}\label{thm:PlantedSharp}
Under the planted dense subgraph model \prettyref{eq:PDSscaling},  if
\begin{align}
\rho f(a,b)>1,
\label{eq:planteddensesubgraphthreshold}
\end{align}
then $\min_{Z^*\in \calZ_n} \pprob{\widehat{Z}_{\SDP}=Z^\ast} =  1 - n^{-\Omega(1)}$ as $n \to \infty$.
\end{theorem}

Next we prove a converse for \prettyref{thm:PlantedSharp} which shows that the recovery threshold achieved by the \SDP relaxation is in fact optimal.
\begin{theorem} \label{thm:PlantedSharpConverse}
If \begin{align}
\rho f(a,b)<1,
\label{eq:planteddensesubgraphthresholdconverse}
\end{align}
then for any sequence of estimators $\widehat{Z}_n$, $\min_{Z^*\in \calZ_n} \pprob{\widehat{Z}_n = Z^\ast } \to 0$.
\end{theorem}

Under the planted dense subgraph model, our investigation of the exact cluster recovery problem thus far in this paper has been focused on the regime where the cluster size $K$ grows \textbf{linearly} with $n$ and $p,q = \Theta(\frac{\log n}{n})$, where the statistically optimal threshold can be attained by \SDP in polynomial time.
However, this need \emph{not} be the case if $K$ grows \textbf{sublinearly} in $n$.
In fact, the exact cluster recovery problem has been studied in \cite{ChenXu14,HajekWuXu14} in the following asymptotic regime:
\begin{equation}
 K =\Theta(n^{\beta}), \; p=cq= \Theta(n^{-\alpha}), \quad n \diverge,
	\label{eq:scaling}
\end{equation}
where $c>1$ and $\alpha, \beta \in (0,1)$ are fixed constants.
The statistical and computational complexities of the cluster recovery problem depend crucially on the value of $\alpha$ and $\beta$ (see \cite[Figure 2]{HajekWuXu14} for an illustration):
\begin{itemize}
\item $\beta> \frac{1}{2}+\frac{\alpha}{2}$: the planted cluster can be perfectly recovered in polynomial-time with high probability
via the \SDP relaxation \prettyref{eq:PDSCVX}.\footnote{In fact, an even looser \SDP relaxation than \prettyref{eq:PDSCVX}
has been shown to exactly recover the planted cluster with high probability for $\beta> \frac{1}{2}+\frac{\alpha}{2}$. See \cite[Theorem 2.3]{ChenXu14}.}
\item $\frac{1}{2}+\frac{\alpha}{4}<\beta<\frac{1}{2}+\frac{\alpha}{2}$: the planted cluster
can be detected in linear time with high probability by thresholding the total number of edges, but it is conjectured
to be computationally intractable to exactly recover the planted cluster.
\item $\alpha<\beta <\frac{1}{2}+\frac{\alpha}{4}$: the planted cluster can be exactly recovered with high probability
via \ML estimation; however, no randomized polynomial-time solver exists conditioned on the planted clique hardness hypothesis.\footnote{Here the planted clique hardness hypothesis refers to the statement that for any fixed constants $\gamma>0$ and $\delta>0$, there exist no randomized polynomial-time tests to distinguish an \ER random graph $\calG(n,\gamma)$ and a planted clique model
which is obtained by adding edges to $k=n^{1/2-\delta}$ vertices chosen uniformly from $\calG(n,\gamma)$ to form a clique. For various hardness results of problems reducible from the planted clique problem, see \cite{HajekWuXu14} and the references within.}
\item $\beta<\alpha$: regardless of the computational costs, no algorithm can exactly recover the planted cluster with vanishing probability of error.
\end{itemize}
Consequently, assuming the planted clique hardness hypothesis, in the asymptotic regime of \prettyref{eq:scaling} when $\alpha \in (0,\frac{2}{3})$ (and, quite possibly, the entire range $(0,1)$), there exists a significant gap between the information limit (recovery threshold of the optimal procedure) and the computational
limit (recovery threshold for polynomial-time algorithms). In contrast, in the asymptotic regime of \prettyref{eq:PDSscaling}, the computational constraint imposes no penalty on the statistical performance, in that the optimal threshold can be attained by \SDP relaxation in view of
\prettyref{thm:PlantedSharp}.

\section{Proofs}
In this section, we give the proofs of our main theorems.
Our analysis of the \SDP relies on two key ingredients: the spectrum
of \ER random graphs and the tail bounds for the binomial distributions,
which we first present.
\subsection{Spectrum of \ER random graph }
Let $A$ denote the adjacency matrix of an \ER random graph $G$, where vertices $i$ and $j$ are connected independently with probability $p_{ij}$.
Then $\expect{A_{ij}}=p_{ij}$. Let $p=\max_{ij} p_{ij}$ and assume $p \ge c_0 \frac{\log n}{n}$
for any constant $c_0>0$.
We aim to show that $\| A-\expect{A}\|\le c' \sqrt{np}$ with high
probability for some constant $c'>0$.
To this end, we establish the following more general result where the entries need not be binary-valued.

\begin{theorem}\label{thm:adjconcentration}
 Let $A$ denote a symmetric and zero-diagonal random matrix, where the entries $\{A_{ij}: i<j\}$ are independent and $[0,1]$-valued.
 Assume that $\expect{A_{ij}} \le p$, where $ c_0 \log n /n \le p \le 1-c_1$ for arbitrary constants $c_0>0$ and $c_1>0$.
  Then for any $c>0$, there exists $c'>0$ such that for any $n \geq 1$, $\pprob{\| A-\expect{A} \| \le c' \sqrt{n p}} \geq 1-n^{-c}$.
\end{theorem}

Let $\calG(n,p)$ denote the \ER random graph model with the edge probability $p_{ij}=p$ for all $i,j$.
Results similar to \prettyref{thm:adjconcentration} have been obtained in \cite{Feige05} for the special case of $\calG(n, \frac{c_0\log n}{n})$
for some \emph{sufficiently large} $c_0$. In fact, \prettyref{thm:adjconcentration} can be proved by strengthening the combinatorial arguments in \cite[Section 2.2]{Feige05}.
Here we provide an alternative proof using results from random matrices and concentration of measures and a seconder-order stochastic comparison argument from \cite{tomozei2014}.

Furthermore, we note that the condition $p= \Omega(\log n /n)$ in \prettyref{thm:adjconcentration} is in fact necessary to ensure that $\|A-\expect{A}\| = \Omega_{\Prob}(\sqrt{n p})$ (see \prettyref{app:adj} for a proof).
The condition $ p \le 1-c_1 $ can be dropped in the special case of $\calG(n,p)$.

\begin{proof}
We first use the second-order stochastic comparison arguments from \cite[Lemma 2]{tomozei2014}.
Since $0 \leq \Expect[A_{ij}] \leq p$, we have $A_{ij}-\expect{A_{ij}} \in [-p, 1]$ for all $i \neq j$ and hence $B_{ij}\triangleq (1-p) (A_{ij}-\expect{A_{ij}}  ) \in [-p, 1-p]$.
Let $C$ denote the adjacency matrix of a graph generated from $\calG(n,p)$.
Then, for any $i,j$,  $B_{ij}$ is stochastically smaller than $C_{ij}-\E{C_{ij}}$ under the convex ordering, i.e.,
$\expect{f (B_{ij}) } \le \expect{f(C_{ij} -\expect{C_{ij}} )}$   for any convex function $f$ on $[-p,1-p]$.\footnote{This follows from $\frac{f(1-p)-f(b)}{1-p-b} \geq f(1-p)-f(-p)$ for any $-p \leq b<1-p$, by the convexity of $f$.}
Since the spectral norm is a convex function and the coordinate random variables are independent (up to symmetry),
it follows that $\Expect[\| {B}\|] \leq \eexpect{ \| {C} - \eexpect{{C}} \| }$ and thus
\begin{align}
\Expect[ \|A-\expect{A}\| ]  =  \frac{1}{1-p} \Expect[\| {B}\|] \leq  \frac{1}{1-p} \eexpect{ \| {C} - \eexpect{{C}} \| } \leq  \frac{1}{c_1} \eexpect{ \| C - \eexpect{C} \| }. \label{eq:boundonA}
\end{align}

We next bound $\eexpect{ \| C - \eexpect{C} \|}$. Let $E=(E_{ij})$ denote an $n\times n$ matrix with independent entries drawn from
$\mu \triangleq \frac{p}{2} \delta_1 + \frac{p}{2} \delta_{-1} + (1-p) \delta_0$, which is the distribution of
a Rademacher random variable multiplied with an independent Bernoulli with bias $p$. Define $E'$ as $E'_{ii}=E_{ii}$ and $E'_{ij}=-E_{ji}$ for all $i \neq j$.
Let $C'$ be an independent copy of $C$. Let $D$ be a zero-diagonal symmetric matrix whose entries are drawn from $\mu$ and $D'$ be an independent copy of $D$.
Let $M=(M_{ij})$ denote an $n\times n$ zero-diagonal symmetric matrix whose entries are Rademacher and independent
from $C$ and $C'$.
We apply the usual symmetrization arguments:
\begin{align}
\Expect[\|C- \Expect[C]\|]
= & ~ \Expect[\|C - \Expect[C']\|] \overset{(a)}\leq \Expect[\|C-C'\|] \overset{(b)}{=} \Expect[\| (C-C') \circ M \| ] \overset{(c)}{\le} 2\Expect[ \|C \circ M\|] \nonumber \\
= & ~ 2 \Expect[\|D\|]  = 2\Expect[ \|D- \Expect[D'] \| ]
\overset{(d)}\leq 2 \Expect[\|D-D'\|] \overset{(e)}{=}2 \Expect[\|E-E'\| ] \overset{(f)}{\leq} 4 \, \Expect[\|E\|] ,\label{eq:sym}
\end{align}
where $(a),(d)$ follow from  the Jensen's inequality; $(b)$ follows because $C-C'$ has the same distribution as $(C-C')\circ M$, where $\circ$ denotes the element-wise product; $(c),(f)$ follow from the triangle inequality; $(e)$ follows from the fact that $D-D'$ has the same distribution as $E - E'$. Then, we apply the result of Seginer \cite{Seginer00} which characterized the expected spectral norm of i.i.d.\ random matrices within universal constant factors. Let $X_j \triangleq \sum_{i=1}^{n} E_{ij}^2$, which are independent $\Binom(n, p)$. Since $\mu$ is symmetric, \cite[Theorem 1.1]{Seginer00} and Jensen's inequality yield
\begin{equation}
\Expect[\|E\|]	\leq \kappa \,\expect{\pth{ \max_{j\in[n]}  X_j }^{1/2} } \leq \kappa \pth{ \expect{\max_{j\in[n]} X_j }}^{1/2}
	\label{eq:seginer}
\end{equation}
for some universal constant $\kappa$.
In view of the following Chernoff bound for the binomial distribution \cite[Theorem 4.4]{Mitzenmacher05}:
\begin{align*}
\prob{X_1 \ge t \log n } \le 2^{- t },
\end{align*}
for all $t \ge 6 np$, setting $t_0=6 \max\{np/\log n, 1\}$ and applying the union bound, we have
\begin{align}
\expect{\max_{j \in [n]} X_j}
= & ~ \int_0^\infty \prob{\max_{j \in [n]} X_j \geq t} \diff t \leq \int_0^\infty (n \, \prob{X_1 \geq t} \wedge 1) \diff t \nonumber \\
\leq & ~ t_0 \log n + n \int_{t_0\log n}^\infty 2^{-t} \diff t \leq (t_0+1) \log n \leq 6(1+2/c_0) np, \label{eq:maxX}
\end{align}
where the last inequality follows from $np \ge c_0\log n$.
Assembling \prettyref{eq:boundonA} -- \prettyref{eq:maxX},
we obtain
\begin{equation}
	\Expect[\|A - \Expect[A]\|] \leq c_2  \sqrt{np},
	\label{eq:spnorm-mean}
\end{equation}
for some positive constant $c_2$ depending only on $c_0, c_1$.
Since the entries of $A - \Expect[A]$ are valued in $[-1,1]$, Talagrand's concentration inequality for 1-Lipschitz convex functions (see, \eg, \cite[Theorem 2.1.13]{tao.rmt}) yields
\[
\prob{\|A - \Expect[A]\| \geq \Expect[\|A - \Expect[A]\|]+t} \leq c_3 \exp(-c_4 t^2)
\]
for some absolute constants $c_3,c_4$, which implies that for any $c>0$,
there exists $c'>0$ depending on $c_0, c_1$, such that $\prob{\|A - \Expect[A]\| \geq c' \sqrt{np}} \leq n^{-c}.$
\end{proof}

\subsection{Tail of the Binomial Distribution}
Let $X \sim \Binom\left( m , \frac{a\log n}{n} \right)$ and $R \sim \Binom\left( m, \frac{b\log n}{n} \right)$ for $m \in \naturals$ and $a,b>0$, where $m =\rho n +o(n) $ for some $\rho>0$ as $n\diverge$.
We need the following tail bounds.

\begin{lemma}[\cite{Abbe14}] \label{lmm:binomialconcentration}
Assume that $a>b$ and $k_n \in \naturals$ such that $k_n= (1+o(1)) \frac{\log n}{\log \log n}$.
Then
\begin{align*}
\prob{X-R \le k_n } \le  n^{ -  \rho \left(\sqrt{a}-\sqrt{b} \right)^2 +o(1)} .
\end{align*}
\end{lemma}
\begin{lemma}\label{lmm:binomialmaxminconcentration}
Let $k_n,k_n' \in [m]$ be such that $k_n= \tau \rho \log n+o(\log n)$ and $k'_n= \tau' \rho \log n+o(\log n)$ for some $0 \leq \tau \leq a$ and $\tau' \geq b$.
Then
\begin{align}
\prob{ X \le k_n } &=  n^{-  \rho \left( a - \tau \log \frac{\eexp a}{\tau} +o(1) \right)   } \label{eq:binomupbound1} \\
\prob{ R \ge  k_n' } & =  n^{-  \rho  \left( b - \tau' \log \frac{\eexp b}{\tau'} +o(1) \right)     } \label{eq:binomupbound2}.
\end{align}
\end{lemma}
\begin{proof}
We use the following non-asymptotic bound on the binomial tail probability \cite[Lemma 4.7.2]{ash-itbook}: For $U \sim \Binom(n,p)$,
\begin{align}
(8 k (1-\lambda))^{-1/2} \exp(- n d(\lambda\|p)) \leq \prob{U \geq k} \leq \exp(- n d(\lambda\|p))
\label{eq:ash1}
\end{align}
where $\lambda = \frac{k}{n} \in (0,1)$ and   $d(\lambda\|p) = \lambda \log \frac{\lambda}{p} + (1-\lambda) \log \frac{1-\lambda}{1-p}$ is the binary divergence function. Then \prettyref{eq:binomupbound2} follows from \prettyref{eq:ash1} by noting that $d(\frac{k_n'}{m} \| \frac{b \log n}{n}) = (b - \tau' \log \frac{b \eexp}{\tau'} + o(1)) \frac{\log n}{n}$.

To prove \prettyref{eq:binomupbound1}, we use the following bound on binomial coefficients \cite[Lemma 4.7.1]{ash-itbook}:
\begin{align}
\frac{\sqrt{\pi}}{2} \leq \frac{\binom{n}{k}}{(2 \pi n \lambda (1-\lambda))^{-1/2}  \exp(n h(\lambda))} \leq 1  .
\label{eq:ash2}
\end{align}
where $\lambda = \frac{k}{n} \in (0,1)$ and  $h(\lambda) = -\lambda \log \lambda - (1-\lambda) \log (1-\lambda)$ is the binary entropy function.
Note that the mode of $X$ is at $\lfloor(m+1)p\rfloor = (a \rho+o(1)) \log n$, which is at least $k_n$ for sufficiently large $n$.
Therefore, $\prob{ X = k  }$ is non-decreasing in $k$ for $k \in [0, k_n]$ and hence
\begin{equation}
\prob{ X = k_n  } \leq \prob{ X \le k_n } \leq k_n \prob{ X = k_n  }	
	\label{eq:mode}
\end{equation}
 where $\prob{ X = k_n } = \binom{m}{k_n} p^k_n (1-p)^{m-k_n}$ and $p=a \log n /n$.
 Applying \prettyref{eq:ash2} to \prettyref{eq:mode} yields
 $$\prob{ X \le k_n } = (\log n)^{O(1)} \exp( - n d(k_n/m \| p)),$$ which is the desired \prettyref{eq:binomupbound1}.
\end{proof}

\subsection{Proofs for the stochastic block model}
	\label{sec:pf-sbm}

The following lemma provides a deterministic sufficient condition for the success of \SDP \prettyref{eq:SBMconvex} in the case $a>b$.
\begin{lemma}\label{lmm:SBMKKT}
Suppose there exist $D^\ast=\diag{d^\ast_i}$ and $ \lambda^\ast \in \reals$ such that $S^* \triangleq D^\ast-A + \lambda^\ast \allones$ satisfies $S^\ast \succeq 0$, $\lambda_2(S^\ast)>0$ and
\begin{align}
S^\ast \sigma^\ast = 0 . \label{eq:SBMKKT}
\end{align}
Then $\widehat{Y}_{\SDP}=Y^\ast$ is the unique solution to \prettyref{eq:SBMconvex}.
\end{lemma}
\begin{proof}
The Lagrangian function is given by
\begin{align*}
L(Y, S, D, \lambda) = \langle A, Y \rangle + \langle S, Y \rangle - \langle D, Y -\identity \rangle - \lambda \Iprod{\allones}{Y},
\end{align*}
where the Lagrangian multipliers are denoted by $S \succeq 0$, $D=\diag{d_i}$, and $\lambda \in \reals$.
Then for any $Y$ satisfying the constraints in \prettyref{eq:SBMconvex},
\begin{align*}
 \Iprod{A}{Y } \overset{(a)}{\le} L(Y, S^\ast, D^\ast, \lambda^\ast) = \Iprod{D^\ast}{I}
=\Iprod{D^\ast}{Y^\ast}=\Iprod{A+S^\ast-\lambda^\ast \allones}{Y^\ast}\overset{(b)}=\Iprod{A}{Y^\ast},
\end{align*}
where $(a)$ holds because $\Iprod{S^\ast}{Y} \ge 0$; $(b)$ holds because $\Iprod{Y^\ast}{S^\ast}=(\sigma^\ast)^\top S^\ast \sigma^\ast =0$ by \prettyref{eq:SBMKKT}.
Hence, $Y^\ast$ is an optimal solution. It remains to establish its uniqueness. To this end, suppose $\tY$ is
an optimal solution. Then,
\begin{align*}
\Iprod{S^\ast}{\tY}=\Iprod{D^\ast-A+ \lambda^\ast \allones}{\tY}\overset{(a)}{=} \Iprod{D^\ast-A}{\tY} \overset{(b)}{=} \Iprod{D^\ast-A}{Y^\ast} {=}\Iprod{S^\ast}{Y^\ast} =0.
\end{align*}
where $(a)$ holds because $\Iprod{\allones}{\tY}=0$; $(b)$ holds because $\Iprod{A}{\tY}=\Iprod{A}{Y^\ast}$ and $\tY_{ii}=Y^*_{ii}=1$ for all $i \in [n]$.
In view of \prettyref{eq:SBMKKT}, since $\tY \succeq 0$, $S^\ast \succeq 0$ with $\lambda_2(S^*)>0$, $\tY$ must be a multiple of $Y^*=\sigma^\ast (\sigma^\ast)^\top$.
Because $\tY_{ii}=1$ for all $i \in [n]$, $\tY=Y^\ast$.
\end{proof}

\begin{proof}[Proof of \prettyref{thm:SBMSharp}]
The theorem is proved first for $a>b$.
Let $D^\ast=\diag{d^\ast_i}$ with
\begin{equation}
d^\ast_i  = \sum_{j=1}^n A_{ij} \sigma^\ast _i \sigma^* _j	
	\label{eq:di-SBM}
\end{equation}
and choose any $\lambda^* \geq \frac{p+q}{2}$. It suffices to show that $S^* = D^\ast-A + \lambda^\ast \allones$ satisfies the conditions in \prettyref{lmm:SBMKKT} with probability $1-n^{-\Omega(1)}$.


By definition, $d^\ast_i \sigma_i^\ast = \sum_{j} A_{ij} \sigma^\ast _j$ for all $i$, \ie, $D^\ast \sigma^\ast =A \sigma^\ast$. Since $ \allones \sigma^\ast=0$, \prettyref{eq:SBMKKT} holds, that is, $S^*\sigma^* = 0$. It remains to verify that $S^\ast \succeq 0$ and $\lambda_2(S^\ast)>0$ with probability
at least $ 1 - n^{-\Omega(1)},$ which amounts to showing that
\begin{equation}
\prob{\inf_{x \Perp \sigma^\ast, \|x\|=1} x^\top S^\ast x  > 0} \geq  1 - n^{-\Omega(1)}.	
	\label{eq:lambda2}
\end{equation}
Note that $\expect{A}= \frac{p-q}{2} Y^\ast + \frac{p+q}{2} \allones - p \identity$ and $Y^\ast= \sigma^\ast (\sigma^\ast)^\top$. Thus for any $x$ such that $x \Perp \sigma^\ast$ and $\|x\|=1$,
\begin{align}
x^\top S^\ast x &= x^\top D^\ast x- x^\top \expect{A} x  +  \lambda^\ast x^\top \allones x - x^\top  \left( A - \expect{A} \right) x \nonumber  \\
&  = x^\top D^\ast  x  - \frac{p-q}{2} x^\top Y^\ast x + \left( \lambda^\ast- \frac{p+q}{2} \right) x^\top \allones x +p  -
x^\top  \left( A- \expect{A} \right) x  \nonumber \\
& \overset{(a)}{\ge} x^\top D^\ast  x +p   -
x^\top  \left( A- \expect{A} \right) x  \ge \min_{i\in[n]} d^\ast_i  + p - \|  A - \expect{A}\|. \label{eq:SBMPSDCheck}
\end{align}
where $(a)$ holds since $\lambda^\ast \ge \frac{p+q}{2}$ and $\iprod{x}{\sigma^\ast}=0$.
It follows from \prettyref{thm:adjconcentration} that $\|  A - \expect{A}\| \leq c' \sqrt{\log n} $ with probability at least $1-n^{-c}$ for some positive constants $c,c'$ depending only on $a$.
Moreover, note that each $d_i$ is equal in distribution to $X-R$, where $X\sim\Binom(\frac{n}{2}-1,\frac{a\log n}{n})$ and $R\sim\Binom(\frac{n}{2},\frac{b\log n}{n})$ are independent. Hence, \prettyref{lmm:binomialconcentration} implies that
\begin{align*}
\prob{X-R \ge \frac{\log n}{\log \log n}} \ge 1- n^{-(\sqrt{a}-\sqrt{b})^2/2+o(1)}.
\end{align*}
Applying the union bound implies that  $\min_{i\in[n]} d^\ast_i  \ge \frac{\log n}{\log \log n}$ holds with probability at least $1 - n^{1-(\sqrt{a}-\sqrt{b})^2/2+o(1)}$. It follows from the assumption $(\sqrt{a}-\sqrt{b})^2>2$ and \prettyref{eq:SBMPSDCheck} that the desired \prettyref{eq:lambda2} holds, completing the proof in the case $a>b$.


For the case $a<b$, we replace the $\argmax$ by $\argmin$ in the \SDP \prettyref{eq:SBMconvex}, which is equivalent to substituting $-A$ for $A$ in the original maximization problem, as well as the sufficient condition in \prettyref{lmm:SBMKKT}. Set the dual variable $d_i^\ast$ according to \prettyref{eq:di-SBM} with $-A$ replacing $A$ and choose any $\lambda^\ast \ge - \frac{p+q}{2}$. Then \prettyref{eq:SBMKKT} still holds
and \prettyref{eq:SBMPSDCheck} changes to $x^\top S^\ast x  \ge \min_{i\in[n]} d^\ast_i  - p - \|  A - \expect{A}\|$, where
$\min_{i\in[n]} d^\ast_i  \ge \frac{\log n}{\log \log n}$ holds with probability at least $1 - n^{1-(\sqrt{a}-\sqrt{b})^2/2+o(1)}$ by \prettyref{lmm:binomialconcentration} and the union bound. Therefore,  in view of \prettyref{thm:adjconcentration} and the assumption $(\sqrt{a}-\sqrt{b} )^2>2$, the desired \prettyref{eq:lambda2} still holds, completing the proof for the case $a<b$.
\end{proof}

\begin{remark}
	For simplicity so far we have focused on the case where $p=\frac{a\log n}{n},q=\frac{b \log n}{n}$ with $a,b$ being fixed constants.
	If we are strictly above the recover threshold, namely, $( \sqrt{a}- \sqrt{b} )^2 > 2 $, \prettyref{thm:SBMSharp} shows that the probability of error is polynomially small in $n$, which cannot be improved by MLE (though a better exponent is conceivable). Now, if we allow $a=a_n$ and $b=b_n$ to vary with $n$, the following result for the optimal estimator (MLE) has been obtained in \cite{Mossel14} that gives the second-order refinement of \prettyref{thm:optimalSBM}: If $a_n,b_n=\Theta(1)$, then clusters can be exactly recovered up to a permutation of cluster indices with
	probability converging to $1$ if and only if
	\begin{equation}
	( \sqrt{a_n}- \sqrt{b_n} )^2 \geq 2 - \frac{ \log \log n }{\log n} + \omega\pth{\frac{1}{\log n}}.
	\label{eq:mls}
\end{equation}
	where for two positive sequences $x_n=\omega(y_n)$ denotes $y_n=o(x_n)$.
	Inspecting the proof of \prettyref{thm:SBMSharp} and replacing \prettyref{lmm:binomialconcentration} by the non-asymptotic version in \cite[Lemma 2]{HajekWuXuSDP15}, one can strengthen the sufficient condition for the success of SDP  \prettyref{eq:SBMconvex}:
	$\min_{Y^* \in \calY_n} \pprob{\widehat{Y}_{\SDP}=Y^\ast} \to 1 $ as $n \to \infty$, provided that
	\begin{equation}
	( \sqrt{a_n}- \sqrt{b_n} )^2 \geq 2 + \frac{ C }{\sqrt{\log n}} + \omega\pth{\frac{1}{\log n}}.
	\label{eq:SDP-2nd}
\end{equation}
	for some universal constant $C$, which is stronger than the optimal condition \prettyref{eq:mls}. It is unclear whether \prettyref{eq:SDP-2nd} is necessary, nor do we know if SDP requires $( \sqrt{a_n}- \sqrt{b_n} )^2-2$ to be positive to succeed.
	\label{rmk:anbn}
\end{remark}

\subsection{Proofs for the planted densest subgraph model}
	\label{sec:pf-pds}
\begin{lemma}\label{lmm:PlanteddenseKKT}
Suppose there exist $D^\ast=\diag{d_i^*} \ge 0$, $B^\ast \in \calS^n$ with $B^\ast \ge 0$, $\lambda^\ast \in \reals$, and
$\eta^\ast \in \reals$ such that $ S^\ast \triangleq D^\ast-B^\ast- A +  \eta^\ast \identity + \lambda^\ast \allones$ satisfies $S^\ast \succeq 0$, $\lambda_2(S^\ast)>0$, and
\begin{align}
S^\ast \xi^\ast &=0, \nonumber \\
d^\ast_i(Z^\ast_{ii}-1) & =0, \quad \forall i, \nonumber \\
B^\ast_{ij} Z^\ast_{ij} &=0, \quad \forall i,j. \label{eq:PlanteddenseKKT}
\end{align}
Then $\widehat{Z}_{\SDP}=Z^\ast$ is the unique solution to \prettyref{eq:PDSCVX}.
\end{lemma}
\begin{proof}
The Lagrangian function is given by
\begin{align*}
L(Z, S, D, B, \lambda, \eta) =  \langle A, Z \rangle + \langle S, Z \rangle - \langle D, Z -\identity \rangle + \Iprod{B}{Z}
-  \eta \left( \Iprod{\identity}{Z} - K \right) - \lambda \left( \Iprod{\allones}{Z} - K^2 \right),
\end{align*}
where $S \succeq 0$, $D=\diag{d_i} \ge 0$, $B \in \calS^n$ with $B \ge 0$, and $\lambda, \eta \in \reals$ are the Lagrangian multipliers.
Then, for any $Z$ satisfying the constraints in \prettyref{eq:PDSCVX},
It follows that
\begin{align*}
\Iprod{A}{Z } & \overset{(a)}{\le}  L(Z, S^\ast, D^\ast, B^\ast, \lambda^\ast, \eta^\ast)
= \Iprod{D^\ast}{I} + \eta^\ast K + \lambda^\ast K^2
\overset{(b)}=\Iprod{D^\ast}{Z^\ast} + \eta^\ast K + \lambda^\ast K^2 \\
& =\Iprod{A+ B^\ast + S^\ast - \eta^\ast \identity - \lambda^\ast \allones  }{Z^\ast}  + \eta^\ast K + \lambda^\ast K^2 \overset{(c)}=
\Iprod{A}{Z^\ast},
\end{align*}
where $(a)$ follows because $\Iprod{S^\ast}{Z} \ge 0$, $\Iprod{D^\ast}{Z-I} \le 0$, and
$\Iprod{B^\ast}{Z} \ge 0$; $(b)$ holds due to $d^\ast_i(Z^\ast_{ii}-1) =0, \forall i$; $(c)$ holds because
$B^\ast_{ij} Z^\ast_{ij} =0, \forall i,j$ and $\Iprod{Z^\ast }{S^\ast}= ( \xi^\ast )^\top S^\ast \xi^\ast =0$.
Hence, $Z^\ast$ is an optimal solution. It remains to establish the uniqueness. To this end, suppose $\tZ$ is another optimal solution. Then,
\begin{align*}
\Iprod{S^\ast}{\tZ}=\Iprod{D^\ast-B^\ast-A+ \eta^\ast \identity + \lambda^\ast \allones }{\tZ} \overset{(a)}{=}\Iprod{D^\ast-B^\ast-A}{\tZ}\overset{(b)}{\le} \Iprod{D^\ast-A}{Z^\ast} {=}\Iprod{S^\ast}{Z^\ast} =0.
\end{align*}
where $(a)$ holds because $\Iprod{\identity}{\tZ}=K$ and $\Iprod{\allones}{\tZ}=K^2$;
$(b)$ holds because $\Iprod{A}{\tZ}=\Iprod{A}{Z^\ast}$, $B^*,\tZ \ge 0$, and $\Iprod{D^*}{\tZ} \leq \sum_{i\in C^*} d^*_{i} = \Iprod{D^*}{Z^*}$ since $d_i^* \geq 0$ and $\tZ_{ii} \leq 1$ for all $i \in [n]$.
Since $\tZ \succeq 0$ and $S^\ast \succeq 0$ with $\lambda_2(S^\ast)>0$, $\tZ$ needs to be a multiple of $Z^\ast=\xi^\ast
(\xi^\ast)^\top$. Then $\tZ=Z^\ast$ since $\Tr(\tZ)=\Tr(Z^\ast)=K$.
\end{proof}

\begin{proof}[Proof of \prettyref{thm:PlantedSharp}]
The theorem is proved first for $a>b$.
Recall $\tau^\ast=\frac{a-b}{\log a- \log b}$ if $a,b>0$ and $ a \neq b$.
Let $\tau^\ast =0$ if $a=0$ or $b=0$. Choose $\lambda^\ast =  \tau^\ast \log n/n $,
$\eta^\ast= \| A -\expect{A}\|$, $D^\ast=\diag{d_i^*}$ with
\begin{align*}
d_i^\ast=\left\{
 \begin{array}{rl}
   \sum_{j \in C^\ast} A_{ij} - \eta^\ast - \lambda^\ast K   & \text{if } i \in C^\ast\\
   0 & \text{otherwise}
    \end{array} \right..
\end{align*}
Define $b^\ast_i \triangleq \lambda^\ast - \frac{1}{K}\sum_{j \in C^\ast} A_{ij}$ for $i \notin C^\ast$. Let $B^*\in \calS^n$ be given by
\[
B^*_{ij} = b^\ast_i \indc{i \notin C^*, j \in C^\ast} + b^\ast_j \indc{i \in C^*, j \notin C^\ast} .
\]
It suffices to show that $(S^\ast, D^\ast, B^\ast)$
satisfies the conditions in \prettyref{lmm:PlanteddenseKKT} with probability at least $ 1 - n^{-\Omega(1)}.	$

By definition, we have $d^\ast_i(Z^\ast_{ii}-1) =0$
and $B^\ast_{ij} Z^\ast_{ij} =0$ for all $i,j \in [n]$. Moreover, for all $i \in C^\ast$,
\begin{align*}
d_i^\ast \xi^\ast_i = d_i^\ast =\sum_{j} A_{ij} \xi^\ast_j  - \eta^\ast - \lambda^\ast K
= \sum_{j} A_{ij} \xi^\ast_j + \sum_j B^\ast_{ij} \xi^\ast_j - \eta^\ast - \lambda^\ast K,
\end{align*}
where the last equality holds because $B^\ast_{ij}=0$ if $(i,j) \in C^\ast \times C^\ast$;
for all $i \notin C^\ast$,
\begin{align*}
\sum_j  A_{ij}  \xi^\ast_j + \sum_j   B^\ast_{ij}\xi^\ast_j   - \lambda^\ast K = \sum_{j \in C^\ast}  A_{ij} + K b_i^\ast -\lambda^\ast K =0,
\end{align*}
where the last equality follows from our choice of $b^\ast_i$. Hence, $D^\ast \xi^\ast=A \xi^\ast + B^\ast \xi^\ast - \eta^\ast \xi^\ast - \lambda^\ast K \mathbf{1}$ and consequently $S^\ast \xi^\ast =0$.

We next show that $D^\ast \ge 0$, $B^\ast \ge 0$ with
probability at least $ 1 - n^{-\Omega(1)}. $ It follows from \prettyref{thm:adjconcentration} that $\eta^\ast \leq c'  \sqrt{ \log n}$
with probability at least $1 - n^{-\Omega(1)} $ for some positive constant $c'$  depending only on $a$. Furthermore,
let $X_i \triangleq \sum_{j \in C^\ast} A_{ij}$. Then $X_i \sim \Binom(K-1, \frac{a\log n}{n})$ if $i \in C^\ast$ and
$\Binom(K, \frac{b\log n}{n})$ otherwise.
We divide the analysis into two separate cases. First consider the case $b=0$, then $X_i=0$ for all $i \notin C^\ast$. Since $\tau^\ast=0$ in this case, $\min_{i \notin C^\ast} b_i^\ast = 0$ holds automatically. For any $i \in C^\ast$, applying \prettyref{lmm:binomialmaxminconcentration} with $\tau=0$ yields
\begin{align*}
\prob{X_i \ge   \frac{\log n} {\log \log n} } & \ge 1- \prob{ X_i  \le \frac{\log n} {\log \log n}  } \ge 1-n^{-\rho a +o(1)}.
\end{align*}
 Applying the union bound implies that $\Prob\{\min_{i \in C^\ast} X_i   \ge  \frac{\log n} {\log \log n}\} \geq 1-n^{1-\rho a+o(1)} \to 1$, because $\rho f(a,0)= \rho a >1$ by the assumption \prettyref{eq:planteddensesubgraphthreshold}. Since $\sqrt{\log n} =o(\frac{\log n}{\log \log n})$ and $\tau^\ast=0$, it follows that with probability at least
$1 - n^{-\Omega(1)}$, $\min_{i \in C^\ast} d_i^\ast \ge 0$ and we are done with the case $b=0$. For $b>0$,
\prettyref{lmm:binomialmaxminconcentration} implies that
\begin{align*}
\prob{X_i \ge  \rho \tau^\ast \log n +  \frac{\log n} {\log \log n} } & \ge 1- n^{-\rho\left( a-\tau^\ast \log \frac{\eexp a}{\tau^\ast} +o(1) \right) }, \quad  \forall i \in C^\ast, \\
\prob{X_i \le  \rho \tau^\ast \log n } & \ge 1-  n^{-\rho\left( b-\tau^\ast \log \frac{\eexp b}{\tau^\ast} +o(1) \right) }, \quad \forall i \notin C^\ast.
\end{align*}
By definition, $f(a,b)= a - \tau^\ast \log \frac{\eexp a}{\tau^\ast}= b - \tau^\ast \log \frac{\eexp b}{\tau^\ast} $ in this case. Applying the union bound implies that with probability at least $1-n^{1-\rho f(a,b)+o(1)}$,
\begin{align*}
\min_{i \in C^\ast} X_i  & \ge \rho \tau^\ast \log n +  \frac{\log n} {\log \log n}, \\
\max_{i \notin C^\ast } X_i & \le \rho \tau^\ast \log n.
\end{align*}
Since $\sqrt{\log n} =o(\frac{\log n}{\log \log n})$ and $\rho f(a,b)>1$ by the assumption \prettyref{eq:planteddensesubgraphthreshold}, it follows that with probability at least $1 - n^{-\Omega(1)},$  $\min_{i \in C^\ast} d_i^\ast \ge 0$ and $\min_{i \notin C^\ast} b_i^\ast \ge 0$.

It remains to verify $S^\ast \succeq 0$ with $\lambda_2(S^\ast)>0$ with probability at least $\geq  1 - n^{-\Omega(1)}$, \ie,
\begin{equation}
\prob{\inf_{x \Perp \sigma^\ast, \|x\|=1} x^\top S^\ast x  > 0} \geq  1 - n^{-\Omega(1)}.		
	\label{eq:lambda2PDS}
\end{equation}
Note that
\begin{align*}
\expect{A}=  (p-q) Z^\ast - p
\begin{bmatrix}
 \identity_{K \times K} & \zeros \\
\zeros & \zeros
\end{bmatrix}
- q \begin{bmatrix}
  \zeros & \zeros \\
\zeros & \identity_{(n-K) \times (n-K)}
\end{bmatrix} + q \allones.
\end{align*}
It follows that for any $x \Perp \sigma^\ast$ and $\|x\|=1$,
\begin{align}
x^\top S^\ast x & ~ = x^\top D^\ast x - x^\top B^\ast x +  (\lambda^\ast -q) x^\top \allones x + p \sum_{i \in C^\ast} x_i^2 + q \sum_{i \notin C^\ast} x_i^2
+ \eta^\ast - x^\top \left(A -\expect{A} \right) x
\nonumber  \\
& ~ \overset{(a)}{=} \sum_{ i \in C^\ast} \left( d_i^\ast + p  \right) x_i^2  + (\lambda^\ast -q) x^\top \allones x
 +  q  \sum_{i \notin C^\ast} x_i^2 + \eta^\ast -x^\top \left(A -\expect{A} \right) x  \nonumber \\
& ~ \ge  \left(  \min_{ i \in C^\ast} d_i^\ast + p \right) \sum_{i \in C^\ast} x_i^2 + (\lambda^\ast -q) x^\top \allones x
  + q \sum_{i \notin C^\ast} x_i^2 + \eta^\ast
 - \| A -\expect{A}\| \nonumber \\
 &~  \overset{(b)}{\ge}  \left(  \min_{ i \in C^\ast} d_i^\ast + p \right) \sum_{i \in C^\ast} x_i^2  + q  \sum_{i \notin C^\ast} x_i^2 \nonumber \\
 &~  \overset{(c)}{\ge}  \min\sth{\min_{ i \in C^\ast} d_i^\ast + p, q},
  \label{eq:semidefinitebound}
 \end{align}
where $(a)$ holds because $B^\ast_{ij}=0$ for all $i, j \notin C^\ast$ and
\begin{align*}
x^\top B^\ast x  = 2 \sum_{i \notin C^\ast} \sum_{j \in C^\ast} x_i x_j B^\ast_{ij}=
2 \sum_{i \notin C^\ast} x_i b^\ast_{i} \sum_{j \in C^\ast} x_j =0 ;
\end{align*}
$(b)$ holds because $\eta^\ast= \| A -\expect{A}\|$ and $\lambda^\ast =\tau^*  \frac{\log n}{n} \ge q= b \frac{\log n}{n}$, since $\log \frac{a}{b} \leq \frac{a}{b} -1$; $(c)$ is due to $\|x\|^2 = 1$.
Notice that we have shown $ \min_{i \in C^\ast} d_i^\ast \ge 0$ with probability at least $1 - n^{-\Omega(1)}.$
Therefore, the desired  \prettyref{eq:lambda2PDS} holds in view of \prettyref{eq:semidefinitebound}, completing the proof in the case $a>b$.

For the case $a<b$, it suffices to modify the above proof by replacing $A$ with $-A$ in the \SDP \prettyref{eq:PDSCVX},
\prettyref{lmm:PlanteddenseKKT}, and the definitions of $d_i^\ast$ and $b_i^\ast$, and choosing
$\lambda^\ast= - \tau^\ast \log n/ n - \log n /(K \log \log n)$, $\eta^\ast =\|A-\E{A}\| + 2q$.
Then \prettyref{eq:PlanteddenseKKT} and \prettyref{eq:lambda2PDS} still hold, and $D^\ast \ge 0$, $B^\ast \ge 0$ with
probability at least $1 - n^{-\Omega(1)}. $ Therefore the theorem follows by applying \prettyref{lmm:PlanteddenseKKT}.
\end{proof}

\begin{proof}[Proof of \prettyref{thm:PlantedSharpConverse}]
To lower bound the worst-case probability of error, consider the Bayesian setting where the planted cluster $C^\ast$ is uniformly chosen among all $K$-subsets of $[n]$ with $K=\lfloor \rho n \rfloor$.
If $a=b$, then the cluster is unidentifiable from the graph.

Next, we prove the theorem first for the case $a>b$.
If $b =0$, then perfect recovery is possible if and only if the subgraph formed by the vertices in cluster, which is $\calG(K,a \log n/n)$, contains no isolated vertex.\footnote{To be more precise, if there is an isolated vertex in the cluster $C^*$, then the likelihood has at least $n-K$ maximizers, which, in turn, implies that the probability of exact recovery for any estimator is at most $\frac{1}{n-K}$.} This occurs with high probability if $\rho a < 1$ \cite{Erdos59}.

Next we consider $a>b>0$.
Since the prior distribution of $C^\ast$ is uniform, the
\ML estimator minimizes the error probability among all estimators and thus we only need to find when the \ML estimator fails.
Let $e(i, S) \triangleq \sum_{j \in S} A_{ij}$ denote the number of edges between vertex $i$ and vertices in $S \subset [n]$.
Let $F$ denote the event that
$\min_{i \in C^*} e(i, C^\ast) < \max_{j \notin C^*} e(j, C^\ast),$
which implies the existence of $i \in C^\ast $ and $j \notin C^\ast $, such that the set $C^*\backslash\{i\}\cup\{j\}$ achieves a strictly higher likelihood than $C^*$.
Hence $\prob{\text{\ML fails} } \ge \prob{F}$. Next we bound $\prob{F}$ from below.

By symmetry, we can condition on $C^*$ being the first $K$ vertices. Let $T$ denote the set of first $ \lfloor \frac{\rho n}{\log^2 n} \rfloor$ vertices.
Then
\begin{equation}
\min_{i \in C^*} e(i, C^*) \leq \min_{i \in T} e(i, C^*) \leq \min_{i \in T} e(i, C^*\backslash T) + \max_{i \in T} e(i, T).	
	\label{eq:wa}
\end{equation}
Let $E_1,E_2,E_3$ denote the event that $\max_{i \in T} e(i,T) < \frac{\log n}{\log \log n}$, $\min_{i \in T} e(i, C^*\backslash T) + \frac{\log n}{\log \log n} \le \tau^\ast \rho \log n$ and $\max_{j \notin C^*} e(j, C^\ast) \ge \tau^\ast \rho \log n$, respectively.
In view of \prettyref{eq:wa}, we have $F \supset E_1 \cap E_2 \cap E_3$ and hence it boils down to proving that $\prob{E_i} \to 1$ for $i=1,2,3$.

In view of the following Chernoff bound for binomial distributions \cite[Theorem 4.4]{Mitzenmacher05}: For $r \ge 1$ and $X \sim \Binom(n,p)$, $\prob{X \ge r np }
\le ( \eexp/r)^{rnp},$
we have
\begin{align*}
\prob{ e(i, T) \ge \frac{\log n}{\log \log n} } \le   \left( \frac{ \eexp\log^2 n}{ \rho a  \log \log n } \right) ^{- \log  n/ \log \log n} = n^{-2+o(1)}.
\end{align*}
Applying the union bound yields
\begin{align*}
\prob{E_1} \ge 1 - \sum_{i \in T} \prob{ e(i, T) \ge \frac{\log n}{\log \log n} } \ge
1 - n^{-1+o(1)}.
\end{align*}
Moreover,
\begin{align*}
\prob{E_2} &
\overset{(a)}{=} 1- \prod_{i\in T} \prob{ e(i, C^*\backslash T)  > \tau^\ast \rho \log n  - \frac{\log n}{\log \log n}} \\
&  \overset{(b)}{=} 1- \left(
1- n^{-  \rho \left( a - \tau^\ast \log \frac{\eexp a}{\tau^\ast} +o(1) \right)  } \right)^{|T|}
 \overset{(c)}{\ge} 1- \exp \left( - n^{1- \rho \left( a - \tau^\ast \log \frac{\eexp a}{\tau^\ast} \right)+o(1) }\right) \overset{(d)}{\to} 1,
\end{align*}
where $(a)$ holds because $\{e(i, C^*\backslash T)\}_{ i \in T}$ are mutually independent; $(b)$ follows from \prettyref{lmm:binomialmaxminconcentration};
$(c$) is due to $1+x \le e^x$ for all $x \in \reals$;
$(d$) follows from the assumption \prettyref{eq:planteddensesubgraphthresholdconverse} that $\rho f(a,b)  = \rho(a - \tau^\ast \log \frac{\eexp a}{\tau^\ast}) < 1$ .
 Similarly,
\begin{align*}
\prob{E_3} &
= 1- \prod_{j\notin C^*} \prob{e(j, C^\ast) < \tau^\ast \rho \log n}  \\
&= 1- \left(
1- n^{- \left( b - \tau^\ast \log \frac{\eexp b}{\tau^\ast} +o(1) \right) }\right)^{n-K}
 \ge 1- \exp \left( - n^{1- \rho \left( b - \tau^\ast \log \frac{\eexp b}{\tau^\ast} \right)+o(1) }\right) \to 1,
\end{align*}
completing the proof in the case $a>b>0$.

Finally, we prove the theorem for the case $a<b$.
Consider the case $a =0$ first. For $j \notin C^\ast$, since $e(j,C^\ast) \sim \Binom(K, \frac{b \log n}{n})$, it follows that
$
\log \prob{e(j, C^\ast) = 0 } = K \log  \left( 1- \frac{b \log n}{n} \right) = -\left( \rho b +o(1) \right) \log n,
$
and thus
\begin{align*}
\prob{\min_{j \notin C^\ast} e(j, C^\ast) =0 } &=
1- \prod_{j \notin C^\ast} \left( 1- \prob{e(j, C^\ast) = 0 } \right) \\
&= 1- \left( 1- n^{-\rho b + o(1) } \right)^{n-K}
\ge 1- \exp ( -  n^{1 -\rho b + o(1) }) \to 1,
\end{align*}
due to the assumption that $\rho f(0,b)= \rho b <1$. Then with probability tending to one, there exists an isolated vertex $j \notin C^*$, in which case the likelihood has at least $K$ maximizers and the probability of exact recovery for any estimator is at most $\frac{1}{K}$.
Next assume that $0<a<b$.
By symmetry, we condition on $C^*$ being the first $K$ vertices. Let $T$ denote the set of first $ \lfloor \frac{\rho n}{\log^2 n} \rfloor$ vertices.
 Redefine $F,E_2,E_3$ as the event that
$\max_{i \in C^*} e(i, C^\ast) > \min_{j \notin C^*} e(j, C^\ast)$, $\max_{i \in T} e(i, C^*\backslash T) > \tau^\ast \rho \log n$ and $\min_{j \notin C^*} e(j, C^\ast) \le \tau^\ast \rho \log n$, respectively.
Then by the same reasoning $\prob{\text{\ML fails} } \ge \prob{F} \geq \prob{E_2\cap E_3}$.
Applying \prettyref{lmm:binomialmaxminconcentration}, we obtain
\begin{align*}
\prob{E_2} &
= 1- \prod_{i\in T} \prob{ e(i, C^*\backslash T)  \le  \tau^\ast \rho \log n } \\
& = 1- \left(
1- n^{-  \rho \left( a - \tau^\ast \log \frac{\eexp a}{\tau^\ast} +o(1) \right)  } \right)^{|T|}
\ge 1- \exp \left( - n^{1- \rho \left( a - \tau^\ast \log \frac{\eexp a}{\tau^\ast} \right)+o(1) }\right) \to 1,
\end{align*}
and similarly,
\begin{align*}
\prob{E_3} &
= 1- \prod_{j\notin C^*} \prob{e(j, C^\ast) >\tau^\ast \rho \log n}  \\
&= 1- \left(
1- n^{- \left( b - \tau^\ast \log \frac{\eexp b}{\tau^\ast} +o(1) \right) }\right)^{n-K}
 \ge 1- \exp \left( - n^{1- \rho \left( b - \tau^\ast \log \frac{\eexp b}{\tau^\ast} \right)+o(1) }\right) \to 1,
\end{align*}
completing the proof for the case $0<a<b$.
\end{proof}

\begin{appendices}
\section{The sharpness of the condition for \prettyref{thm:adjconcentration}}
\label{app:adj}
Consider the case where $A$ is the adjacency matrix of $\calG(n,p)$.
We show that if $p =o(\frac{\log n}{n})$ and $p=n^{-1+o(1)}$,
then
\begin{equation}
\prob{\| A-\expect{A}\| \geq c \sqrt{ \frac{\log n}{\log \left( \log n / (np) \right) }}} \to 1
\label{eq:adj-anti}
\end{equation}
for some constant $c$ and, consequently,
$\| A-\expect{A}\| / \sqrt{np} \to \infty$ in probability.
To this end, note that $\| A-\expect{A}\| \geq \max_{i \in [n]} \|(A-\expect{A}) e_i\|$, where $\{e_i\}$ denote the standard basis.
Without loss of generality, assume $n$ is even.  Then
by focusing on the upper-right part of $A$, we have
\[
\| A-\expect{A} \|^2 \geq \max_{i \in [n/2] } \sum_{j > n/2} (A_{ij}-p)^2 =  np^2/2 + (1-2p) \max_{i \in [n/2]} \sum_{j > n/2} A_{ij} \geq (1-2p) \max_{i \in [n/2]} X_i,
\]
where $X_i \iiddistr \Binom(n/2,p)$.
Using the inequality $\binom{n}{k} \ge \left( \frac{n}{k} \right)^k$ for $ k \ge 0$
with the convention that $0^0=0$, 
\begin{align*}
\prob{X_1 \geq k} \ge \prob{X_1 = k} = \binom{n/2}{k} p^k (1-p)^{n/2-k} \ge \left( \frac{np}{2k} \right)^k (1- p)^{n/2}.
\end{align*}
Since $\log (1-x) \ge -2x$ for $x \in [0, 1/2]$, it follows that
\begin{align}
-\log \prob{X_1 \geq k}  \le  k \log\left( \frac{2k}{np} \right) - \frac{n}{2} \log (1-p) \le
k \log\left( \frac{2k}{np} \right) + np. \label{eq:binomiallowerbound}
\end{align}
Plugging $k^\ast \triangleq  \lfloor \frac{\log n}{\log \left( \log n / (np) \right) } \rfloor$ into \prettyref{eq:binomiallowerbound}, 
we get that
\begin{align}
-\log \prob{X_1 \geq k^\ast } & \le
\log n + np - \left( \frac{\log n} {\log \left( \log n / (np) \right) } -1 \right) \log \left( \frac{\log \left( \log n / (np) \right)}{2} \right) \\
 & \le \log n-\log \log n, \label{eq:probabilitylowerbound}
\end{align}
where the last inequality follows due to $np=o(k^\ast)$ since $np = o(\log n)$ and  $\log (\log n/ (np) ) = o(\log n)$ since $np=n^{o(1)}$,
By the independence of $\{X_i, i \in [n/2] \}$, we have
\begin{align*}
\prob{ \max_{i \in [n/2]} X_i  < k^\ast} &= \prod_{i=1}^{n/2} \prob{ X_i < k^\ast } = (  1- \prob{ X_1 \ge k^* })^{n/2} \le \exp \left( -\frac{n}{2} \prob{X_1 \geq k^\ast } \right) \le \frac{1}{\sqrt{n}},
\end{align*}
where the last inequality follows in view of \prettyref{eq:probabilitylowerbound}.
\end{appendices}

\bibliographystyle{abbrv}
\bibliography{../../../one_community/graphical_combined}

\end{document}